%% file: main.tex
\documentclass[11pt, a4paper, copyright, goog]{google}

\usepackage[authoryear, sort&compress, round]{natbib}
\input{math_commands.tex}

\input{macros}

\usepackage{hyperref}
\usepackage{url}
\usepackage{natbib} 
\usepackage{booktabs}
\usepackage{wrapfig}
\usepackage{enumitem}
\usepackage{comment}

\usepackage{fancybox}
\usepackage{booktabs}
\usepackage{makecell}
\usepackage{appendix}
\usepackage{hyperref}
\usepackage{tcolorbox}
\usepackage[draft]{fixme}
\usepackage{chngcntr}
\usepackage{longtable}
\usepackage{multirow}

\usepackage{amsfonts}
\usepackage{titlesec}
\usepackage{amsmath}
\usepackage{graphicx}
\usepackage{listings}
\usepackage{xcolor}
\usepackage{listings}
\usepackage{xcolor}

\usepackage{algorithm}

\usepackage[font=small,labelfont=bf]{caption}
\usepackage[capitalise]{cleveref}

\usepackage{mathtools}
\usepackage{algpseudocode}
\usepackage{placeins}

\newcommand{\fig}[1]{Fig.~\ref{#1}}
\newcommand{\sect}[1]{Sec.~\ref{#1}}
\newcommand{\tab}[1]{Tab.~\ref{#1}}

\usepackage{authblk}
\RequirePackage{authblk}
\renewcommand\Affilfont{\normalfont\fontsize{8}{10}\selectfont}
\makeatletter
\renewcommand\AB@affilsepx{, \protect\Affilfont}
\makeatother

\uselogo{} 

\title{Multi-agent cooperation through in-context co-player inference}

\author[$\star$,1]{Marissa A. Weis}
\author[$\star$,1]{Maciej Wołczyk}
\author[1]{Rajai Nasser}
\author[1]{Rif A. Saurous}
\author[1,2]{Blaise Agüera y Arcas}
\author[1]{João~Sacramento}
\author[1]{Alexander Meulemans}

\affil[1]{Google, Paradigms of Intelligence Team}
\affil[2]{Santa Fe Institute}
\affil[$\star$]{Equal contribution}

\begin{abstract}
Achieving cooperation among self-interested agents remains a fundamental challenge in multi-agent reinforcement learning. Recent work showed that mutual cooperation can be induced between ``learning-aware'' agents that account for and shape the learning dynamics of their co-players. However, existing approaches typically rely on hardcoded, often inconsistent, assumptions about co-player learning rules or enforce a strict separation between ``naive learners'' updating on fast timescales and ``meta-learners'' observing these updates. Here, we demonstrate that the in-context learning capabilities of sequence models allow for co-player learning awareness without requiring hardcoded assumptions or explicit timescale separation. We show that training sequence model agents against a diverse distribution of co-players naturally induces \textit{in-context best-response} strategies, effectively functioning as learning algorithms on the fast intra-episode timescale. We find that the cooperative mechanism identified in prior work—where vulnerability to extortion drives mutual shaping—emerges naturally in this setting: in-context adaptation renders agents vulnerable to extortion, and the resulting mutual pressure to shape the opponent's in-context learning dynamics resolves into the learning of cooperative behavior. 
Our results suggest that standard decentralized reinforcement learning on sequence models combined with co-player diversity provides a scalable path to learning cooperative behaviors.
\end{abstract}

\begin{document}

\maketitle

\section{Introduction}

The development of foundation model agents is rapidly shifting the landscape of artificial intelligence from isolated systems to interacting autonomous agents \citep{xi_rise_2023, park_generative_2023, aguera2026silicon}. As these sequence-model-based agents are deployed in increasingly complex environments, they inevitably face multi-agent interactions where outcomes depend on interactions of multiple entities. Because these interactions frequently involve competing goals, 
ensuring that self-interested agents robustly cooperate in mixed-motive settings remains an important open challenge, even as individual agent capabilities have grown significantly.

Decentralized Multi-Agent Reinforcement Learning (MARL) addresses the problem of learning to interact with other agents while only having access to local observations. However, decentralized MARL is challenging due to two primary factors: equilibrium selection and non-stationarity of the environment \citep{shoham_multiagent_2008, hernandez-leal_survey_2017}. In general-sum games, many Nash equilibria may exist, and agents independently optimizing their own rewards frequently converge to suboptimal outcomes, such as mutual defection in social dilemmas \citep{claus_dynamics_1998, foerster_learning_2018}. Furthermore, from the perspective of a single agent, the environment is non-stationary because other agents are simultaneously learning and adapting their policies \citep{hernandez-leal_survey_2017}. Since standard single-agent reinforcement learning (RL) algorithms assume stationarity, they often fail to learn effective policies in these decentralized settings \citep{foerster_learning_2018, claus_dynamics_1998}.

To address this non-stationarity, \textit{co-player learning awareness} enables agents to anticipate the learning dynamics of other agents and shape their co-players' learning toward more beneficial equilibria~\citep{foerster_learning_2018, lu_model-free_2022, duque2024advantage, aghajohari2024loqa, cooijmans_meta-value_2023, balaguer_good_2022, khan_scaling_2024, willi_cola_2022, xie_learning_2021, meulemans2024multi, aghajohari2024best, segura2025opponent, piche2025learning}. These approaches generally fall into two categories. The first explicitly models the co-player's learning update, estimating a shaping gradient by differentiating through the opponent's update step \citep{foerster_learning_2018, duque2024advantage, aghajohari2024loqa, cooijmans_meta-value_2023, willi_cola_2022, aghajohari2024best, piche2025learning}. However, this requires rigid assumptions about the opponent's learning rule and creates inconsistencies if the opponent is also learning-aware. The second category implicitly learns to shape opponents by extending the RL time horizon to encompass multiple update steps of the co-player \citep{lu_model-free_2022, khan_scaling_2024, meulemans2024multi, segura2025opponent}. While effective, this requires a separation of agents into ``naive learners'' (who update parameters frequently) and ``meta-learners'' (who update slowly), effectively treating the interaction as a meta-learning problem \citep{bengio_learning_1990, schmidhuber_evolutionary_1987, hochreiter_learning_2001}.

\citet{meulemans2024multi} describe a three-step mechanism explaining why co-player learning awareness leads to the learning of cooperative behaviors among self-interested agents:
\begin{enumerate}[leftmargin=*]
    \item \textbf{Extortion of naive learners:} The optimal strategy against a naive learner (an agent updating its policy to maximize rewards on a fast timescale) is extortion \citep{press_iterated_2012}. A learning-aware meta-agent shapes the interaction so the naive learner updates its policy towards more cooperation, allowing the meta-agent to exploit the resulting behavior.%
    \item \textbf{Mutual extortion leads to cooperation:} When two agents with such extortionate capabilities face each other, their attempts to shape the learning of one another result in both agents learning more cooperative strategies.
    \item \textbf{Heterogeneity is key:} Consequently, cooperation emerges when agents are trained in a mixed population of naive learners and learning-aware agents. Interactions with naive learners provide the gradient pressure to learn extortion (avoiding mutual defection), while interactions with learning-aware agents refine this into mutual cooperation.
\end{enumerate}

We argue that the complex mechanisms employed by current co-player learning-aware methods, such as explicit naive learners and meta learners, or differentiating through co-players' learning updates, are unnecessary for learning cooperative behaviors. %
We hypothesize that training sequence model agents via decentralized MARL against a diverse distribution of co-players naturally yields \textit{in-context best-response} policies. These policies exhibit goal-directed adaptation through in-context learning within a single episode. Crucially, we show that this acts as a functional drop-in replacement for the ``naive learner'' parameter updates of prior work. Because in-context learning occurs on a fast timescale within the episode, agents become susceptible to extortion by other learning agents using in-weight updates. Consequently, the cooperative gradient dynamics identified by \citet{meulemans2024multi} emerge: gradients incentivizing the extortion of in-context learners pull agents away from pure defection, while mutual extortion gradients drive them toward cooperation.%

Our contributions are as follows. We introduce a decentralized MARL setup where sequence model agents are trained against a mixed pool of diverse co-players and demonstrate that this training distribution induces strong in-context co-player inference capabilities and thereby the mutual extortion pressures leading to cooperation. We show that this setup leads to robust cooperation in the Iterated Prisoner's Dilemma without the distinction between meta and inner trajectories, or assumptions about opponent learning rules. By bridging in-context learning and co-player learning-awareness, we provide a scalable path toward cooperative multi-agent systems using standard sequence modeling and RL. We introduce a new RL method that leverages self-supervised learning of predictive sequence models, which is well-suited to learn the in-context best-response policies required for the mixed pool training. We provide a theoretical characterization of the training equilibrium of this method, and relate it to Nash equilibria and subjective embedded equilibria \citep{meulemans2025embedded}.

\section{Problem setup and methods}

\textbf{Partially observable stochastic games}. 
We formalize the multi-agent interaction as a partially observable stochastic game \citep[POSG;][]{kuhn_extensive_1953} of $N$ agents. Each agent $i$ receives at each timestep an observation $o^i_t \in \mathcal{O}^i$ and reward $r^i_t \in \mathcal{R}^i$, and executes an action $a^i_t \in \mathcal{A}^i$, with $\mathcal{O}^i$, $\mathcal{R}^i$ and $\mathcal{A}^i$ being finite sets. Policies are conditioned on the interaction history $x^i_{\leq t} = \{(o^i_k, a^i_{k-1}, r_{k-1}^i)\}_{k=1}^{t}$. We denote the policy of agent $i$ as $\pi^i(a^i_t \mid x^i_{\leq t}; \phi^i)$, parameterized by $\phi^i$.

\textbf{The iterated prisoner's dilemma}. 
We focus on the Iterated Prisoner's Dilemma (IPD), a canonical model for studying cooperation among self-interested agents \citep{rapoport_prisoners_1974,axelrod_evolution_1981}. In each round $t$, two agents choose simultaneously to cooperate (C) or defect (D), i.e., $a_t^i \in \{\mathrm{C}, \mathrm{D}\}$, receiving payoffs as detailed in~\tab{tab:payoff_ipd}. This structure creates a social dilemma: in a single-shot game, mutual defection is the unique Nash equilibrium, even though mutual cooperation yields higher global and individual returns. While the infinitely iterated game allows for cooperative Nash equilibria, converging to these equilibria via decentralized reinforcement learning remains challenging \citep{foerster_learning_2018, claus_dynamics_1998}. For computational tractability, we approximate the infinite horizon with a fixed horizon of $T=100$ steps, which is sufficient for the small-scale policy networks used in this work to approximate infinite-horizon behavior.

\textbf{Mixed pool training}. 
To induce robust in-context inference capabilities, we train agents within a mixed population rather than against a single fixed opponent. The training pool consists of (i)~\textbf{Learning Agents} which use a sequence model policy that processes the full episode history~$x^i_{\leq t}$ and whose parameters are learned during training, and (ii)~static \textbf{Tabular Agents} parameterized by a 5-dimensional vector, defining the probability of cooperating in the initial state and in response to the four possible joint action outcomes of the previous turn $(a_{t-1}^i, a_{t-1}^{-i})$.
During training, a learning agent plays $50\%$ of its episodes against another learning agent and $50\%$ against a tabular agent sampled uniformly from the parameter space. Crucially, agents do not receive agent identifiers; they must infer the nature and strategy of their opponent solely from the interaction history $x^i_{\leq t}$.

We investigate two learning algorithms for the learning agents in our pool:

\textbf{Independent A2C.} We employ Advantage Actor-Critic (A2C)~\citep{mnih_asynchronous_2016} as a standard decentralized model-free RL method. Each agent independently optimizes its policy parameters $\phi^i$ to maximize its own expected return, treating the other agents as part of the environment.

\textbf{Predictive Policy Improvement (PPI).} 
We introduce a model-based algorithm that leverages a sequence model predicting the joint sequence of actions, observations, and rewards, serving simultaneously as a world model and a policy prior. This method is a variation of Maximum A-Posteriori Policy Optimization~\citep[MPO]{abdolmaleki2018maximum}, inspired by the MUPI framework for multi-agent learning \citep{meulemans2025embedded}, and enables efficient learning of in-context inference mechanisms through self-supervised training. 
Each iteration consists of (i) gathering data with the improved policy and (ii) retraining the sequence model on the newly gathered data, similar to classical policy iteration. We define the improved policy $\pi^i(a^i|x^i_{\leq t})$ as follows:
\begin{equation}
    \pi^i(a^i|x^i_{\leq t}) \propto p^i_{\phi^i}(a^i|x^i_{\leq t}) \cdot \exp\left(\beta \hat{Q}_{p^i}(x^i_{\leq t},a^i)\right),
\end{equation}
where $\beta$ is an inverse temperature hyperparameter. The action value $\hat{Q}_p(h,a)$ is estimated via Monte Carlo rollouts performed within the sequence model $p_\phi$. 
We deploy this improved policy $\pi^i(a^i|x^i_{\leq t})$ in the games interacting with other agents, collecting a new batch of trajectories. We end the iteration by retraining the sequence model $p^i_{\phi^i}$ on all accumulated trajectory batches of the current and previous iterations, distilling the improved behavior of $\pi^i$ into the parameters $\phi^i$. We initialize the sequence model $p_\phi$ by pretraining on interactions between randomly sampled tabular agents. Refer to App.~\ref{app:methods} for the implementation details, App.~\ref{app:ppi_derivation} for a theoretical derivation and motivation of PPI, and App.~\ref{app:predictive-equilibrea} for a theoretical analysis of the equilibrium behavior of PPI agents.

\section{Results}

\begin{figure}[t]
    \centering
    \includegraphics[width=\textwidth]{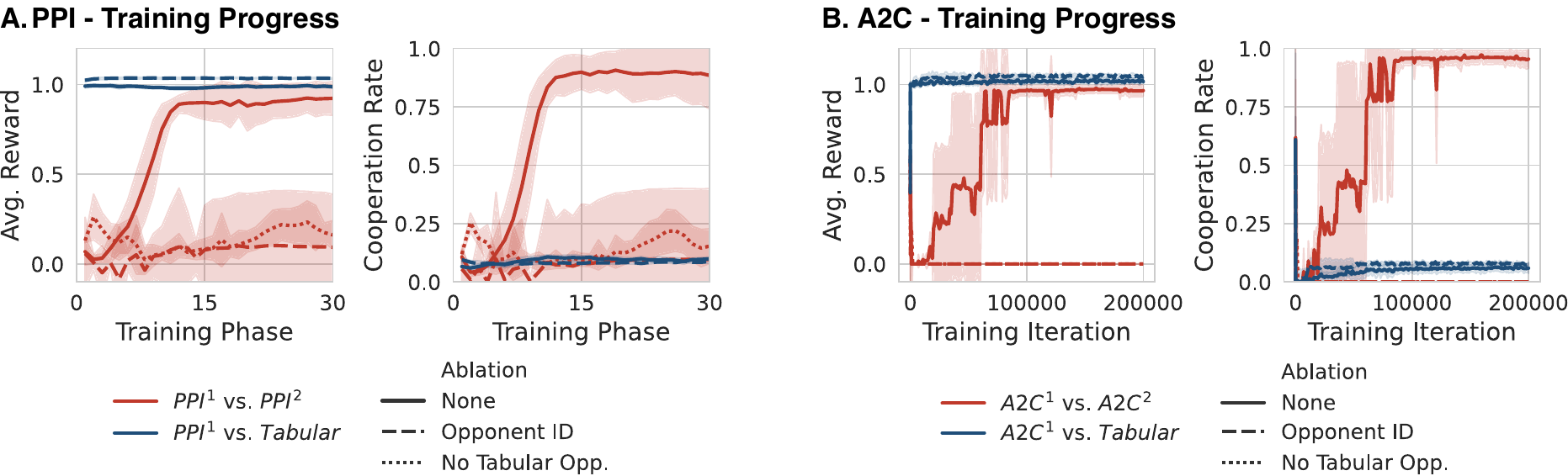}
    \caption{\textbf{Mixed training leads to robust cooperation.} RL agents trained against a mix of tabular policies and learning agents converge to cooperation (solid lines). \textbf{Ablations:} Agents trained purely against other learning agents (dotted lines) or with access to explicit co-player identifications (dashed lines) converge to defection, highlighting that in-context inference is a critical factor for the learning of cooperative behaviors with standard decentralized MARL. Error bars indicate standard deviation across 10 random seeds.}
    \label{fig:mixed_training}
\end{figure}

Our central hypothesis is that training the learning agents against a diverse distribution of co-players necessitates the development of two distinct capabilities: (i) \emph{inferring} the co-player's policy from interaction history, and (ii) \emph{adapting} to a best response within a single episode. We posit that this \emph{in-context best-response policy} makes the agent vulnerable to extortion, reproducing the ``naive learner'' dynamics described by \citet{meulemans2024multi}. This leads to learning pressures towards extortion policies, and subsequently, the mutual extortion between learning agents drives the agents toward cooperative policies. Interestingly, in this setup, the learning agents simultaneously occupy two roles traditionally separated in the literature: they are ``naive learners'' on the fast timescale (via in-context learning) and ``learning-aware agents'' on the slow timescale (via weight updates).

In this section, we first demonstrate that mixed-pool training indeed leads to robust cooperation without explicit time-scale separations or meta-gradient machinery. We then dissect the underlying mechanism, showing that (1) mixed pool training induces in-context best-response policies, (2) these policies are vulnerable to extortion, and (3) mutual extortion pressures resolve into learning cooperative behaviors.

\subsection{Mixed training induces robust cooperation}\label{sec:mixed}

As shown in Figure~\ref{fig:mixed_training}, both PPI and A2C agents trained in the  mixed pool setup converge to cooperation in IPD. To verify this stems from the dynamics of in-context opponent inference, we perform two ablations:
 \textbf{(1) Explicit Identification:} We condition the policy on the opponent's policy parameters (for tabular opponents) or identity flag (for other learning agents) at the start of the episode, removing the need for in-context opponent inference. \textbf{(2) No mixed pool training:} We train agents solely against a single other learning agent (without the tabular agent pool or structured pretraining). Without diverse opponents, agents have no incentive to develop general-purpose in-context learning mechanisms.
In both ablations, agents collapse to mutual defection (c.f.~\fig{fig:mixed_training}; dashed and dotted curves). This confirms that in-context learning mechanisms—induced by the necessity to identify diverse opponents—are a critical factor enabling cooperative outcomes. Refer to App.~\ref{app:abl} for the ablation details.

\begin{figure}[t]
    \centering
    \includegraphics[width=0.91\textwidth]{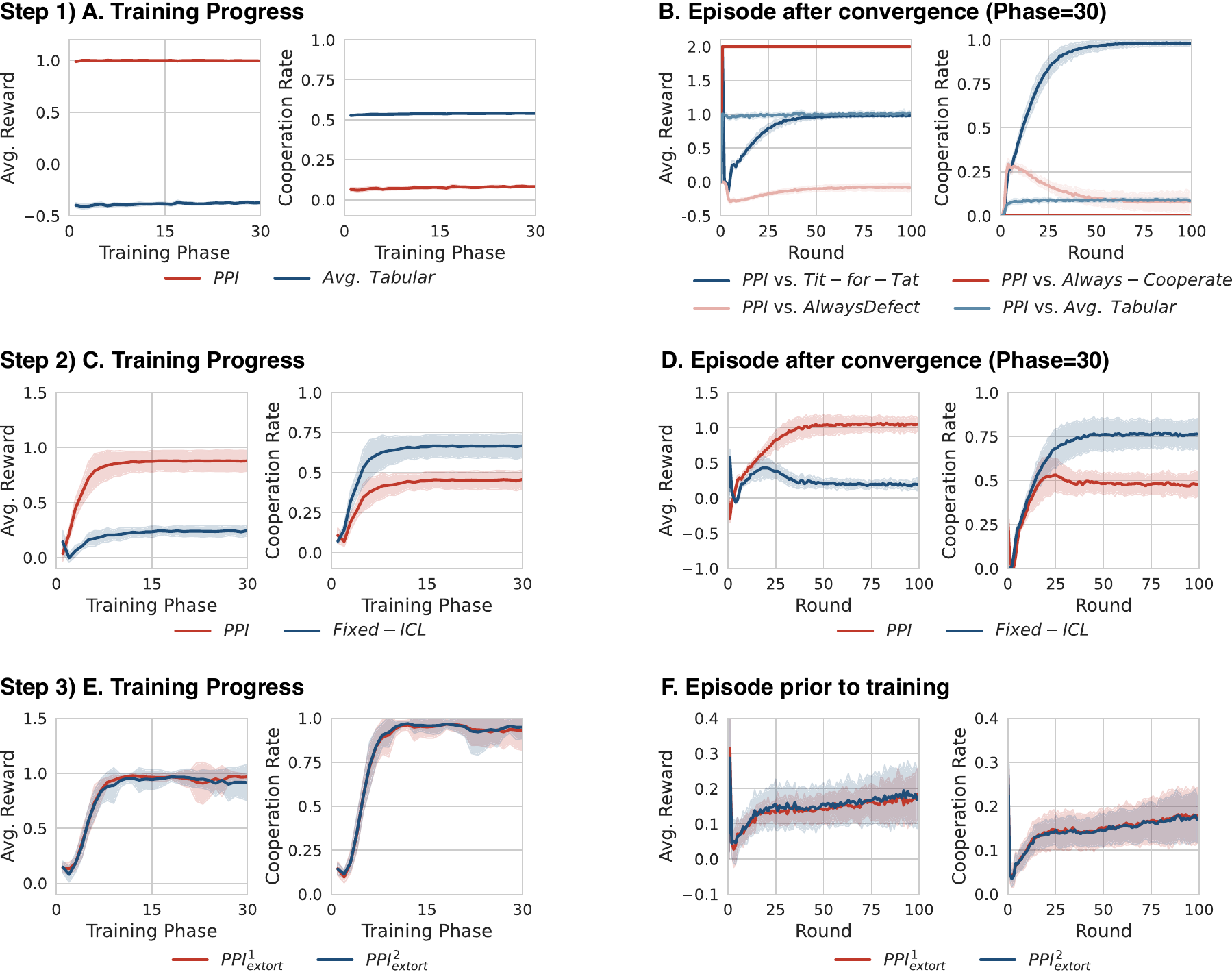}
    \caption{\textbf{A--B: Emergence of in-context best response.} Performance of PPI agents (trained against random tabular opponents) when evaluated against specific fixed strategies. The agents demonstrate in-context learning, identifying the opponent and converging to the best response within the episode. \textbf{C--D: Learning to extort in-context learners.} Agents trained against a ``Fixed In-Context Learner'' (an agent pre-trained in Step~1 to best-respond to tabular policies) learn to extort it. The RL agent achieves a higher share of the reward by exploiting the in-context adaptation of its opponent. \textbf{E--F: From mutual extortion to cooperation.} When two agents initialized with extortion policies (from Step~2) play against each other, their mutual attempts to extort their co-player result in the shaping of each other's policy towards more cooperative behavior, both within episodes through in-context learning (\textbf{F}) and across episodes through in-weight learning (\textbf{E}). Error bars indicate standard deviation across 10 random seeds.}
    \label{fig:best_response}
\end{figure}

\FloatBarrier

\subsection{Mechanism analysis: From in-context learning to cooperation}

We now deconstruct the learning of cooperative behavior into three distinct steps, validating the causal chain from diversity to in-context learning, to extortability, and finally to cooperation.

\textbf{Step 1: Diversity induces in-context best-response mechanisms.}
First, we verify that training against the tabular pool cultivates in-context learning. We evaluate a PPI agent trained solely against the tabular agents pool. Figure~\ref{fig:best_response}\textbf{B} plots the agent's performance against specific tabular policies over the course of an episode. The agent rapidly adapts to the best response for the specific opponent. This confirms the emergence of \emph{in-context best-response mechanisms} that perform goal-directed adaptation on the fast timescale of the episode.

\textbf{Step 2: In-context learners are vulnerable to extortion.}
Next, we establish that such in-context best-response policies are susceptible to shaping by other co-players. We freeze the agent from Step~1, termed the "Fixed In-Context Learner" (Fixed-ICL), and train a new PPI agent solely against it. The new agent learns to \emph{extort} the Fixed-ICL policy~(\fig{fig:best_response}\textbf{C}\&\textbf{D}) \citep{press_iterated_2012}. By exploiting the Fixed-ICL's tendency to adapt, the new agent forces it into unfair cooperation, maximizing the new agent's own reward at the expense of the Fixed-ICL. This confirms that goal-directed adaptation within the episode provides the necessary gradient signal for opponents to learn extortionate behaviors via weight updates.

\textbf{Step 3: Mutual extortion drives cooperation.}
We initialize two agents with the extortion policies learned in Step~2 and train them against each other. Within an episode, both extortion policies shape each others in-context learning dynamics into more cooperative behavior~(\fig{fig:best_response}\textbf{F}). This push towards more cooperation is then picked up by the parameter updates, further driving both policies towards cooperative behavior~(\fig{fig:best_response}\textbf{E}), mirroring the "mutual shaping" effect observed in explicit learning-aware methods \citep{lu_model-free_2022, meulemans2024multi}.

\textbf{Step 4: Synthesis in mixed populations.}
Mixed-pool training combines these dynamics by forcing agents to maintain in-context adaptation for tabular opponents, which renders them vulnerable to mutual extortion by other learners, ultimately driving the learning agents toward cooperation through mutual extortion~(\sect{sec:mixed}; \fig{fig:mixed_training} \& \fig{fig:mixed_training_in_eps}). Figure \ref{fig:a2c_best_response} in Appendix~\ref{app:a2c_results} shows similar results for A2C learning agents.

\section{Conclusion}
In this work, we have demonstrated that the complex machinery of explicit co-player learning-awareness—such as meta gradients or rigid timescale separation—is not required to learn cooperative behaviors in general-sum games. Instead, we found that simply training agents against a diverse distribution of co-players suffices to induce in-context best-response strategies. This in-context learning renders agents susceptible to shaping and consequently driving them toward cooperative behaviors through mutual extortion dynamics. Crucially, this result bridges the gap between multi-agent reinforcement learning and the training paradigms of modern foundation models. Since foundation models naturally exhibit in-context learning and are trained on diverse tasks and behaviors, our findings suggest a scalable and computationally efficient path for the emergence of cooperative social behaviors using standard decentralized learning techniques.

\section*{Acknowledgments}
We would like to thank Guillaume Lajoie, Angelika Steger and the Google Paradigms of Intelligence team for feedback and insightful discussions.

\input{main.bbl}
\include{appendix}

\end{document}

%% file: math_commands.tex
\usepackage{amsmath,amsfonts,bm}

\def\eqref#1{equation~\ref{#1}}

\def\1{\bm{1}}

\DeclareMathAlphabet{\mathsfit}{\encodingdefault}{\sfdefault}{m}{sl}
\SetMathAlphabet{\mathsfit}{bold}{\encodingdefault}{\sfdefault}{bx}{n}

\DeclareMathOperator*{\argmin}{arg\,min}

%% file: macros.tex
\usepackage{thm-restate, amsmath, amssymb, amsfonts, mathtools, amsthm}

\usepackage{tikz}
\usepackage{xcolor} %

\usepackage{comment}

\theoremstyle{plain}
\newtheorem{theorem}{Theorem}[section]

\newtheorem{lemma}[theorem]{Lemma}

\newtheorem{corollary}[theorem]{Corollary}

\theoremstyle{definition}
\newtheorem{definition}[theorem]{Definition}

\DeclareMathSymbol{\smin}{\mathbin}{AMSa}{"39}

\newcommand{\bbP}{\mathbb{P}}

\newcommand{\expect}[2]{\mathbb{E}_{#1}\left[ #2 \right]}

\newcommand{\kl}[2]{\mathrm{KL}\left(#1 \mid \mid #2\right)}

\newcommand{\thetavec}{\bm{\theta}}

%% file: appendix.tex
\FloatBarrier
\newpage
\appendix

\section{Additional details on methods}\label{app:methods}

\subsection{Partially observable stochastic games}
We formalize the multi-agent interaction as a partially observable stochastic game \citep[POSG;][]{kuhn_extensive_1953} defined by the tuple $(\mathcal{I}, \mathcal{S}, \mathcal{A}, P_t, P_r, P_i, \mathcal{O}, P_o, \gamma, T)$. Here, $\mathcal{I} = \{1,\dots,n\}$ is the set of $n$ agents. At each time step $t$, the environment is in state $s_t \in \mathcal{S}$. Agents simultaneously select actions from the joint action space $\mathcal{A} = \times_{i \in \mathcal{I}} \mathcal{A}^i$, transitioning the environment according to $P_t(S_{t+1} \mid S_t, A_t)$. The initial state is sampled from $P_i(s_0)$. Each agent $i$ receives a reward $r^i_t$ from the joint factorized distribution $P_r = \times_{i \in \mathcal{I}}P_r^i(r^i \mid s, a)$, and an observation $o^i_t$ from the observation space $\mathcal{O}=\times_{i \in \mathcal{I}} \mathcal{O}^i$ via the distribution $P_o(o_t \mid s_t, a_{t-1})$. %
We denote the discount factor by $\gamma$ and the horizon by $T$.
We use the superscript $i$ to denote variables specific to agent $i$, and $-i$ for the remaining agents. %
Policies are conditioned on the interaction history $x^i_{\leq t} = \{(o^i_k, a^i_{k-1}, r_{k-1}^i)\}_{k=1}^{t}$. We denote the policy of agent $i$ as $\pi^i(a^i_t \mid x^i_{\leq t}; \phi^i)$, parameterized by $\phi^i$.

\subsection{Environment}

\begin{wraptable}{r}{6cm}
\renewcommand{\arraystretch}{1.5}
\caption{Single-round IPD payoff matrix}
\label{tab:payoff_ipd}
\centering
\begin{tabular}{cccc}
 & & \multicolumn{2}{c}{\textbf{Player 2}} \\
 & \multicolumn{1}{c|}{} & \textbf{C} & \textbf{D} \\ \cline{2-4}
\multirow{2}{*}{\textbf{Player 1}} & \multicolumn{1}{c|}{\textbf{C}} & (1, 1) & (-1, 2) \\
 & \multicolumn{1}{c|}{\textbf{D}} & (2, -1) & (0, 0) \\
\end{tabular}
\end{wraptable}

\paragraph{Iterated Prisoners Dilemma (IPD)} In each round both agents can output two possible actions: cooperate ($C$) and defect ($D$). As such, the environment emits five possible observations: the initial observation $s_0$ and four observations based on the actions the two players took in the previous round: $(C, C), (C, D), (D, C), (D, D)$. The state~$s_t$ is then comprised of all past observations~$o_{\leq t}$.  While the tabular agents are only conditioned on the latest observation $o_t$, the PPI and A2C agents leverage the full history~$x_{\leq t}$. Each game consists of 100 rounds. Each agent observes the state of the previous round from a first person view, i.e., its own action is enumerated first. In every round, each agent receives a reward following the payoff matrix in \tab{tab:payoff_ipd}.

\begin{algorithm}[t]
\caption{Predictive Policy Improvement}
\label{alg:seq_model_rl}
\begin{algorithmic}[1]
\Require Initial sequence model $p_{\phi_0}$, reinforcement learning environment~$\mathcal{E}$, number of iterations~$N_{\text{iter}}$, number of training epochs~$N_{\text{epochs}}$, number of samples~$N_{\text{samples}}$, initial dataset~$\mathcal{D}_0$

\For{$k = 1$ to $N_{\text{iter}}$}
    \State Initialize weights $\phi_k$ of $p_{\phi_k}$ randomly
    \For{$e = 1$ to $N_{\text{epochs}}$} \Comment Step 1: Train sequence model
        \State Update parameters of $p_{\phi_k}$ using $\mathcal{D}_{k-1}$ to minimize loss function $L_{train}$ in Eq.~\ref{eq:loss}
    \EndFor

    \State Initialize empty dataset $\mathcal{R}_k$.
    \For{$r = 1$ to $N_{\text{samples}}$} \Comment Step 2: Collect game trajectories
        \State Reset environment $\mathcal{E}$.
        \State Generate a sequence of actions/observations using $p_{{\phi_k}}$ within $\mathcal{E}$.
        \State Collect trajectory $\tau_r = (o_0, r_0, a_0, o_1, r_1, a_1, \dots)$ from $\mathcal{E}$.
        \State Add $\tau_r$ to $\mathcal{R}_k$.
    \EndFor
    \State Set $\mathcal{D}_{k}\leftarrow \mathcal{D}_{k-1}\cup\mathcal{R}_k$ for the next iteration's training.
\EndFor
\end{algorithmic}
\end{algorithm}

\subsection{Agent implementations}

\subsubsection{PPI agents}
Predictive Policy Improvement (PPI) agents, our practical approximation of embedded Bayesian agents~\citep{meulemans2025embedded}, combine a learned sequence model with a planning-based policy improvement mechanism.

\paragraph{Sequence Model Architecture.} The sequence model is a Gated Recurrent Unit (GRU) with a 128-dimensional hidden state. Inputs—comprising observations, actions, and rewards—are processed via modality-specific linear layers to project them into a shared 32-dimensional embedding space; observations and actions are one-hot encoded prior to projection. These embeddings serve as inputs to the GRU, and we apply the Swish activation function \citep{ramachandran2017searching} on the output. Distinct linear output heads decode the hidden states to predict future tokens for each modality.

\paragraph{Training Objectives.} We train the sequence model iteratively for 30 phases. In each phase, the model parameters $\phi$ are re-initialized and trained on a dataset of interaction histories $\mathcal{D} = \{x^{(n)}\}^{N}_{n=1}$ to minimize the next-token prediction loss:

\begin{align}\label{eq:loss}
L_{\text{train}} &= \lambda_{\text{obs}} L_{\text{obs}} + \lambda_{\text{act}} L_{\text{action}} + \lambda_{\text{rew}} L_{\text{reward}} \,,\\
L_{\text{obs}} &= -\frac{1}{NT} \sum_{n=1}^{N} \sum_{t=1}^{T} \log p_\phi(o^{(n)}_t | x_{\leq t-1}^{(n)}) \,,\\
L_{\text{reward}} &= -\frac{1}{NT} \sum_{n=1}^{N} \sum_{t=1}^{T} \log p_\phi(r^{(n)}_t | x_{\leq t-1}^{(n)}, o^{(n)}_{t})\,,\\
L_{\text{action}} &= -\frac{1}{NT}  \sum_{n=1}^{N} \sum_{t=1}^{T} \log p_\phi(a^{(n)}_t | x_{\leq t-1}^{(n)}, o^{(n)}_{t}, r^{(n)}_{t})\,.
\end{align}
$\mathcal{D}$ comprises of the interaction histories from all previous and current phases. This is a common strategy in performative prediction \citep{perdomo2020performative} to ensure more stable training of the prediction model.

We model $p_\phi(a_t \mid x_{\leq t})$ and $p_\phi(o_t \mid x_{<t}, a_{t-1})$ using a categorical distribution, yielding a standard categorical cross-entropy loss and we model $p_\phi(r_t \mid x_{<t}, a_{t-1}, o_t)$ with a normal distribution with fixed variance, yielding the mean-square error loss $(r - \hat{r})^2$.
In each phase, we sample $20\,000$ trajectories, which are concatenated with samples from previous phases for joint training of the sequence model.
Optimization is performed using AdamW \citep{loshchilov2017decoupled} (learning rate $10^{-4}$, weight decay $10^{-2}$, $\beta_1=0.9$, $\beta_2=0.98$) for 10 epochs with a batch size of 256. Gradients are clipped at a norm of 1.0.

\paragraph{Pre-training} The sequence model is pretrained on an initial dataset $\mathcal{D}_0$ of $200\,000$ sample trajectories of two random tabular agents playing IPD against each other for 100 rounds using the same training hyperparameters as outlined above.

\paragraph{Inference} During deployment, the agent estimates Q values by performing Monte Carlo roll-outs for 15 rounds into the future using the learned sequence model as a simulator. The final action selection follows a policy $\pi(a | x_{\leq t})$ that re-weights the model's prior probability $p(a | x_{\leq t};\phi)$ by the estimated value $\hat{Q}^p(x_{\leq t}, a)$ derived from the roll-outs:
\begin{equation} \label{eq:ppi}
\pi(a | x_{\leq t}) = \frac{1}{Z}p(a | x_{\leq t};\phi) \exp(\beta \hat{Q}^p(x_{\leq t}, a))\,.
\end{equation}
We use $\beta=0.01$ for all experiments.

\subsubsection{Model-free agent}
\paragraph{Architecture} We implement an Advantage Actor-Critic (A2C) agent~\citep{mnih_asynchronous_2016} using a GRU-based sequence model with the same configuration as for the PPI agents. The GRU takes as input the history of observations of previous rounds and outputs the next action. The GRU is further augmented with a linear output head to estimate the value function $V(x)$. During training, we estimate the advantage $A(x_{\leq t}, a_t)$ using bootstrapped temporal-difference errors:
$$A(x_{\leq t}, a_t) = r_t + \gamma V(x_{\leq t+1}) - V(x_{\leq t})\,.$$
The model parameters are updated to minimize the combined policy gradient and value estimation loss:
$$\begin{aligned}
L = \sum_{t=1}^{T} \bigg( 
&- \log \pi(a_t | x_{\leq t}) A(x_{\leq t}, a_t) \\
&+ c_v \left( r_t + \gamma V(x_{\leq t+1}) - V(x_{\leq t}) \right)^2 \\
&+ c_e \sum_i \pi(a_t^i | x_{\leq t})\log \pi(a_t^i | x_{\leq t}) \bigg)\,,
\end{aligned}$$

where $c_v, c_e$ are hyperparameters representing, correspondingly, the value function and entropy training coefficients.

\paragraph{Training}

To get comparable results, we follow the A2C training protocol from~\citet{meulemans2024multi} including the value function estimation, Generalized Advantage Estimation~\citep{schulman2015high}, advantage normalization and reward scaling. See Appendix A of \citet{meulemans2024multi} for details. 

For each experiment, we perform a hyperparameter search over the learning rate, GAE lambda, advantage normalization, reward scaling and entropy regularization. We report the hyperparameters corresponding to the best-performing setting in Table~\ref{tab_app:hp_a2c}. 

\begin{table}[htpb!]
\caption{A2C hyperparameters \label{tab_app:hp_a2c}}
\centering
\begin{tabular}{@{}lllll@{}}
\toprule
\textbf{RL Hyperparameter} & \textbf{Step 1} & \textbf{Step 2} & \textbf{Step 3} & \textbf{Step 4} \\ \midrule
\texttt{advantages\_normalization}      & True                    & False                  & True                         & True                      \\
\texttt{batch size}      & 2048                    & 2048                  & 4096                         & 4096                      \\
\texttt{reward\_rescaling}              & $0.2$                    & $0.05$                  & $0.02$                         & $0.02$                      \\
\texttt{value\_discount ($\gamma$)}     & $0.99$                    & $0.99$                  & $0.99$                         & $0.99$                      \\
\texttt{td\_lambda ($\lambda_\text{td}$)}  & $0.99$                     & $1.0$                   & $0.95$                          & $1.0$                       \\
\texttt{gae\_lambda ($\lambda_\text{gae}$)} & $0.99$                     & $1.0$                   & $0.95$                          & $1.0$                       \\
\texttt{value\_coefficient}             & $0.5$                     & $0.5$                   & $0.5$                          & $0.5$                       \\
\texttt{entropy\_reg}                   & $0.001$                     & $0.001$                   & $0.001$                          & $0.01$                       \\
\texttt{optimizer}                      & Adam                      & Adam                    & Adam                           & Adam                        \\
\texttt{adam\_epsilon}                  & $0.00001$                 & $0.00001$               & $0.00001$                      & $0.00001$                  \\
\texttt{learning\_rate}                 & $0.005$                   & $0.005$                 & $0.0005$                        & $0.001$                     \\
\texttt{max\_grad\_norm}                & $1.0$                     & $1.0$                   & $1.0$                          & $1.0$                       \\
\bottomrule
\end{tabular}
\end{table}

\subsubsection{Tabular agents}
Tabular agents employ a memory-1 policy defined by five parameters: the cooperation probabilities conditional on the previous outcome ($cc, cd, dc, dd$) and the initial state ($s_0$). Each parameter is initialized from a uniform distribution $\mathcal{U}(0, 1)$.

\subsection{Ablations}\label{app:abl}
\subsubsection{Policy conditioning} For the ``Opponent ID'' ablation (\fig{fig:mixed_training}), we prepend the observation sequence $\mathbf{x}_{\leq t}$ with a conditioning vector $\mathbf{z}$ representing the opponent's identity. For tabular agents, $\mathbf{z}$ is defined as the flattened vector of log-probabilities across all possible observations $o \in \mathcal{O}$:
$$ \mathbf{z} = \big( \log \pi(a|o) \big)_{o\in \mathcal{O}, a \in \mathcal{A}} $$
where $\mathcal{O} = \{(C,C), (C,D), (D,C), (D,D), \text{Start}\}$. For A2C and PPI agents, $\mathbf{z} = \textbf{0}$.

\subsubsection{No mixed pool training} For the ``No Tabular Opponents'' ablation (\fig{fig:mixed_training}), we remove the tabular opponents from the mixed agent pool for both PPI and A2C experiments. For PPI, we additionally change the pretraining data distribution $D_0$ to not include tabular agents but instead consist of purely random action sequences with the corresponding rewards.

\section{Additional results}
\label{app:results}

\subsection{In-episode trajectories for mixed pool training}

Figure~\ref{fig:mixed_training_in_eps} shows the performance of PPI and A2C within a single episode during early training in the mixed pool setting (c.f. \sect{sec:mixed}), i.e., for $\text{phase}=8$ for PPI and $\text{training iteration}=70k$ for A2C, showing the emergence of in-context opponent inference and an initial gradient towards cooperation against other learning agents.

\begin{figure}[t]
    \centering
    \includegraphics[width=\textwidth]{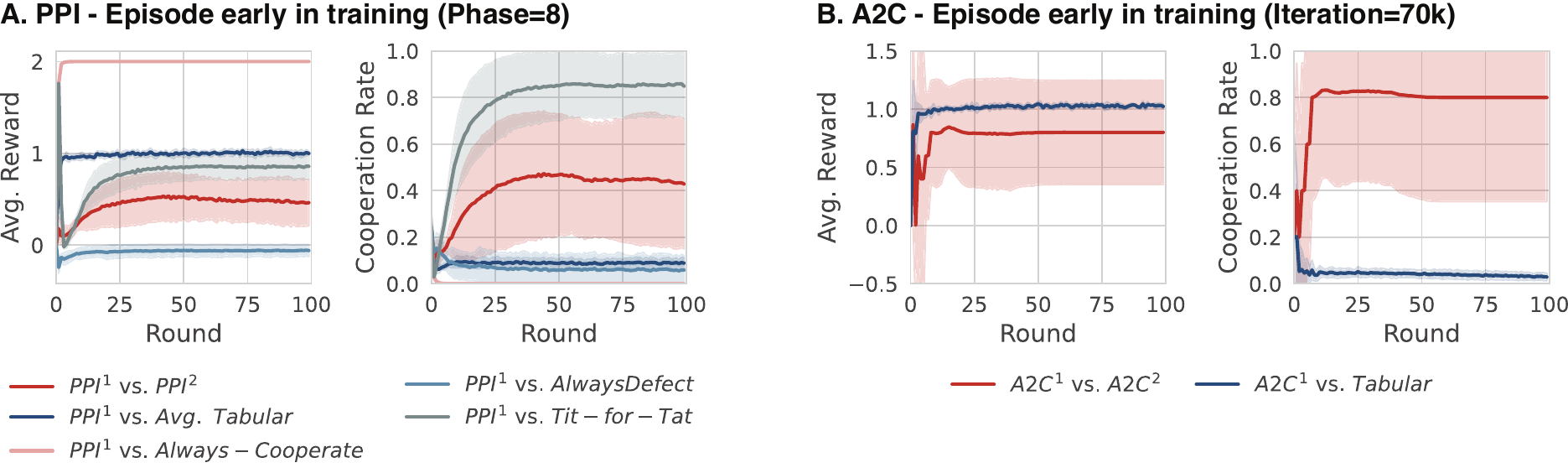}
    \caption{\textbf{Emergence of best-response in mixed training.} We plot within-episode performance of models trained in Figure~\ref{fig:mixed_training} before convergence. We observe that both A2C and PPI try to extort their counterpart at the beginning of the episode which subsequently leads to increased levels of cooperation. At the same time, identifying the opponent as a non-tit-for-tat-like tabular policy leads to high defection ratio. Error bars indicate standard deviation across 10 random seeds.}
    \label{fig:mixed_training_in_eps}
\end{figure}

\subsection{Additional results on A2C}
\label{app:a2c_results}

\begin{figure}[t]
    \centering
    \includegraphics[width=\textwidth]{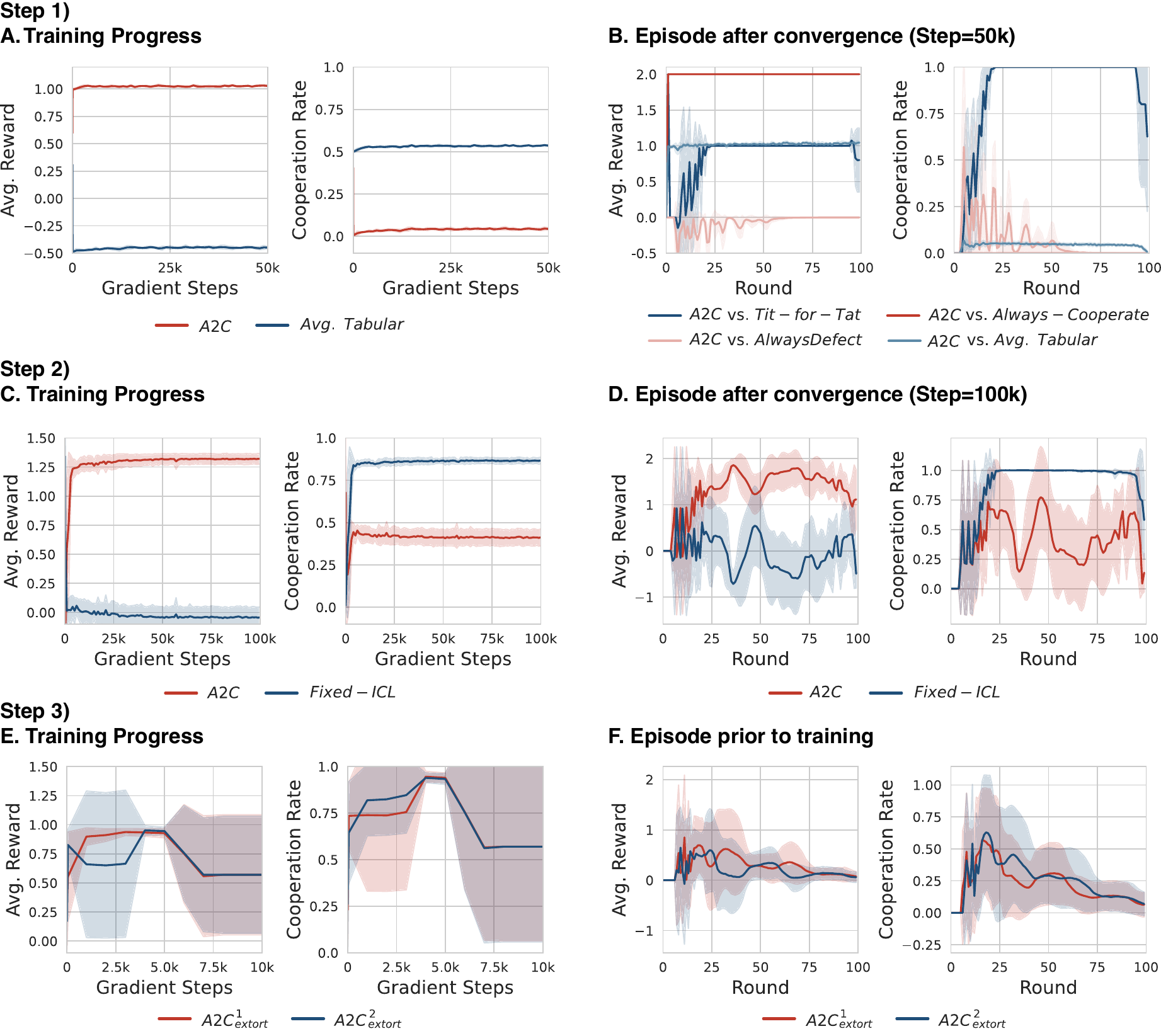}
    \caption{\textbf{A-B: Emergence of in-context best response} Performance of A2C trained against random tabular opponents and evaluated after convergence on a set of specific static policies. We denote the final agent as ``Fixed In-Context Learner''. \textbf{C-D: Learning to extort in-context learners.} Performance of a randomly initialized A2C agent against the Fixed In-Context Learner. \textbf{E-F: From mutual extortion to cooperation.} Two A2C extortion agents initially converge to cooperation when playing against each other, but with time they might collapse to mutual defection depending on the random seed. Error bars correspond to standard deviation over 5 random initializations.}
    \label{fig:a2c_best_response}
\end{figure}

Figure~\ref{fig:a2c_best_response} shows A2C-based results, corresponding to the PPI results presented in Figure \ref{fig:best_response} in the main text. In Step 1, we observe that an A2C agent learns to implement best response against a variety of tabular agents, same as for PPI. In Step 2, however, we observe that the newly trained A2C agent manages to get a higher reward playing against the Fixed-ICL baseline than the PPI agent (correspondingly, $\sim1.25$ vs. $\sim0.9$). This can either be caused by (i) the PPI Fixed-ICL policy being harder to exploit or (ii) A2C finding a better exploiter policy. The irregular shape of the exploitation dynamics in Figure \ref{fig:a2c_best_response}D suggests that the A2C exploiter agent learned a complex adversarial strategy against the A2C Fixed-ICL policy. In contrast, the PPI extortion policy of Figure \ref{fig:best_response}D seems to be a more regular extortion policy. Finally, in Step 3, the A2C agents initially move towards cooperation but due to training instability they might still turn back to defection depending on the seed.  

\FloatBarrier

\section{Derivation of Predictive Policy Improvement (PPI)}
\label{app:ppi_derivation}

In this section, we provide a formal derivation of the Predictive Policy Improvement (PPI) algorithm. PPI is inspired by the theoretically grounded MUPI framework \citep{meulemans2025embedded}, and is closely related to Maximum a Posteriori Policy Optimization \citep[MPO;][]{abdolmaleki2018maximum}. PPI departs from standard MPO by replacing the separate value function and self-model of MPO with a single sequence model trained in a self-supervised fashion to predict actions, observations and rewards. This model serves simultaneously as a world model and a policy prior, leveraging the generative capabilities of sequence models for value estimation and policy representation.

\subsection{Objective: The Variational Lower Bound}
We consider an agent optimizing its policy $\pi$ to maximize the expected return $V(\pi) = \mathbb{E}_{\tau \sim \mathbb{P}_{\pi}} \left[\sum_{t=0}^T \gamma^t r_t\right]$. To avoid notational clutter, we omit the agent-specific superscripts, as this derivation applies equally to the single-agent setting. We introduce a parameterized sequence model $p_{\phi}(a \mid x_{\leq t})$ which acts as a behavioral prior or self-model over the interaction history $x_{\leq t}$. We define a surrogate objective $J$ by penalizing the KL-divergence between the behavioral policy $\pi$ and the prior $p_{\phi}$:
\begin{equation}
    J(\pi, \phi) = \mathbb{E}_{\tau \sim \mathbb{P}_{\pi}} \left[ \sum_{t=0}^T \gamma^t r_t - \alpha \mathrm{KL}\left(\pi(\cdot \mid x_{\leq t}) \,||\, p_{\phi}(\cdot \mid x_{\leq t})\right) \right].
\end{equation}
Since $\mathrm{KL}(\cdot||\cdot) \geq 0$, $J(\pi, \phi)$ is a strict lower bound on $V(\pi)$, with equality at $\pi = p_{\phi}$. We optimize this bound via coordinate ascent on $\pi$ (the non-parametric policy) and $\phi$ (the parametric sequence model).

\subsection{Step 1: Non-parametric Policy Improvement w.r.t. $\pi$}
Optimizing $J(\pi, \phi)$ w.r.t. $\pi$ for a fixed $\phi$ is a full-fledged optimal control problem, which generally lacks an analytical solution and is therefore ill-suited for direct non-parametric policy improvement. Instead, we use a first-order approximation of $J(\pi, \phi_{k})$ around $\pi = p_{\phi_k}$, where $p_{\phi_k}$ is the self-model trained on the dataset gathered by deploying the previous policy $\pi_{k-1}$:
\begin{equation}\label{eqn:j_bar}
    \begin{split}
    \bar{J}(\pi, \phi_k) &= \sum_{t=1}^T\mathbb{E}_{x_{\leq t} \sim \mathbb{P}_{p_{\phi_k}}} \Big[ \mathbb{E}_{a \sim \pi(\cdot \mid x_{\leq t})} \left[ Q^{p_{\phi_k}}(x_{\leq t}, a) \right] - \alpha \mathrm{KL}\left(\pi(\cdot \mid x_{\leq t}) \,||\, p_{\phi_k}(\cdot \mid x_{\leq t})\right) \\
    &\quad - V^{p_{\phi_k}}(x_{\leq t}) \Big] + J(p_{\phi_k}, \phi_k).
    \end{split}
\end{equation}

Note that here, the Q-value is equal to the unregularized value $Q^{p_{\phi_k}}(x_{\leq t}, a) =  \mathbb{E}_{\tau_{>t} \sim \mathbb{P}_{p_{\phi_k}}(\cdot \mid x_{\leq t}, a)} \left[\sum_{t'=t}^T \gamma^{t'-t} r_{t'} \right]$, as all KL terms evaluate to zero under the prior. The crucial difference between $J$ and $\bar{J}$ is that the expectation over histories in $\bar{J}$ does not depend on the policy $\pi$ being optimized, which permits a closed-form solution for $\arg\max_{\pi} \bar{J}$.

We proceed to show that $\bar{J}$ is indeed a first-order approximation to $J$ around $p_{\phi_k}$ via the following two lemmas.

\begin{lemma}
    $\bar{J}(p_{\phi_k}, \phi_k) = J(p_{\phi_k}, \phi_k)$
\end{lemma}
\begin{proof}
    It is easy to see that the terms of \eqref{eqn:j_bar} inside the expectation cancel out when $\pi = p_{\phi_k}$, leaving only $J(p_{\phi_k}, \phi_k)$.
\end{proof}

\begin{lemma}
    $\nabla_\pi \bar{J}(\pi, \phi_k) \vert_{\pi = p_{\phi_k}} = \nabla_\pi J(\pi, \phi_k) \vert_{\pi = p_{\phi_k}}$
\end{lemma}

\begin{proof}
We analyze the functional derivatives of both objectives with respect to the policy distribution $\pi(a \mid x_{\leq t})$ evaluated at a specific history $x_{\leq t}$ and action $a$. 

First, consider the surrogate objective $\bar{J}(\pi, \phi_k)$. Because the expectation over histories is fixed to the prior distribution $\mathbb{P}_{p_{\phi_k}}$ and thus does not depend on the optimization variable $\pi$, the functional derivative is straightforward. Applying the product rule to the logarithmic term, the functional derivative with respect to the local action probability $\pi(a \mid x_{\leq t})$ is:
\begin{equation}
    \frac{\delta \bar{J}(\pi, \phi_k)}{\delta \pi(a \mid x_{\leq t})} = \mathbb{P}_{p_{\phi_k}}(x_{\leq t}) \left( Q^{p_{\phi_k}}(x_{\leq t}, a) - \alpha \log \frac{\pi(a \mid x_{\leq t})}{p_{\phi_k}(a \mid x_{\leq t})} - \alpha \right).
\end{equation}
Evaluating this derivative at the prior $\pi = p_{\phi_k}$, the logarithmic term vanishes (since $\log 1 = 0$), yielding:
\begin{equation}
    \left. \frac{\delta \bar{J}(\pi, \phi_k)}{\delta \pi(a \mid x_{\leq t})} \right|_{\pi = p_{\phi_k}} = \mathbb{P}_{p_{\phi_k}}(x_{\leq t}) \left( Q^{p_{\phi_k}}(x_{\leq t}, a) - \alpha \right).
\end{equation}

Next, differentiating the true objective $J(\pi, \phi_k)$ is more involved because $\pi$ dictates the history visitation distribution $\mathbb{P}_{\pi}(x_{\leq t})$. We define the regularized Q-function, $Q^\pi_{\mathrm{reg}}(x_{\leq t}, a)$, which captures the expected return including all \emph{future} KL penalties, but excluding the immediate penalty at time $t$:
\begin{equation}
    Q^\pi_{\mathrm{reg}}(x_{\leq t}, a) = \mathbb{E}_{\tau_{>t} \sim \mathbb{P}_\pi(\cdot \mid x_{\leq t}, a)} \left[ \sum_{k=t}^T \gamma^{k-t} R_k - \alpha \sum_{k=t+1}^T \gamma^{k-t} \mathrm{KL}\left(\pi(\cdot \mid x_{\leq k}) \,||\, p_{\phi_k}(\cdot \mid x_{\leq k})\right) \right].
\end{equation}
Using this, the value of a specific history is:
\begin{equation}
    V^\pi_{\textrm{reg}}(x_{\leq t}) = \sum_{a'} \pi(a' \mid x_{\leq t}) \left( Q^\pi_{\mathrm{reg}}(x_{\leq t}, a') - \alpha \log \frac{\pi(a' \mid x_{\leq t})}{p_{\phi_k}(a' \mid x_{\leq t})} \right).
\end{equation}

To find the functional derivative of the global objective $J$ with respect to the local policy $\pi(a \mid x_{\leq t})$, we apply the continuous extension of the Performance Difference Lemma \citep{kakade2002approximately}. This theorem establishes that the indirect effect of the policy on the visitation distribution $\mathbb{P}_{\pi}(x_{\leq t})$ yields a net zero contribution to the gradient. Consequently, the derivative isolates the state visitation probability multiplied by the local derivative of the value function:
\begin{equation}
    \frac{\delta J(\pi, \phi_k)}{\delta \pi(a \mid x_{\leq t})} = \mathbb{P}_{\pi}(x_{\leq t}) \frac{\partial V^\pi(x_{\leq t})}{\partial \pi(a \mid x_{\leq t})}.
\end{equation}
Taking the partial derivative of $V^\pi(x_{\leq t})$ yields:
\begin{equation}
    \frac{\delta J(\pi, \phi_k)}{\delta \pi(a \mid x_{\leq t})} = \mathbb{P}_{\pi}(x_{\leq t}) \left( Q^\pi_{\mathrm{reg}}(x_{\leq t}, a) - \alpha \log \frac{\pi(a \mid x_{\leq t})}{p_{\phi_k}(a \mid x_{\leq t})} - \alpha \right).
\end{equation}

Finally, we evaluate this true derivative at the prior policy $\pi = p_{\phi_k}$. Three simplifications occur:
\begin{itemize}
    \item The history visitation distribution matches the prior: $\mathbb{P}_{\pi}(x_{\leq t}) = \mathbb{P}_{p_{\phi_k}}(x_{\leq t})$.
    \item The immediate KL penalty evaluates to zero: $\log 1 = 0$.
    \item Because the policy perfectly matches the prior at all future timesteps, all future KL penalties evaluate to zero. Consequently, the regularized Q-function smoothly collapses to the unregularized Q-function of the prior: $Q^\pi_{\mathrm{reg}}(x_{\leq t}, a) = Q^{p_{\phi_k}}(x_{\leq t}, a)$.
\end{itemize}
Applying these simplifications yields:
\begin{equation}
    \left. \frac{\delta J(\pi, \phi_k)}{\delta \pi(a \mid x_{\leq t})} \right|_{\pi = p_{\phi_k}} = \mathbb{P}_{p_{\phi_k}}(x_{\leq t}) \left( Q^{p_{\phi_k}}(x_{\leq t}, a) - \alpha \right).
\end{equation}

Since the functional derivatives of both $J$ and $\bar{J}$ evaluated at $\pi = p_{\phi_k}$ perfectly coincide, it follows that $\nabla_\pi \bar{J}(\pi, \phi_k) \vert_{\pi = p_{\phi_k}} = \nabla_\pi J(\pi, \phi_k) \vert_{\pi = p_{\phi_k}}$, concluding the proof.
\end{proof}

\paragraph{Optimizing $\bar{J}$.}
Optimizing $\bar{J}(\pi, \phi_k)$ w.r.t. $\pi$ for fixed $\phi_k$ has the well-known Boltzmann policy as solution:
\begin{equation}
        \pi^*(a \mid x_{\leq t}) = \frac{p_{\phi_k}(a \mid x_{\leq t})}{Z(x_{\leq t})} \exp\left(\beta Q^{p_{\phi_k}}(x_{\leq t}, a)\right),
    \end{equation}
with the inverse temperature $\beta=\frac{1}{\alpha}$. We treat $\beta$ as a fixed hyperparameter defining a trust region around $p_{\phi_k}$ where $\bar{J}$ is a sufficiently accurate approximation of $J$.

\subsection{Comparison with MPO and Sequence-Model Value Estimation}
While PPI shares the coordinate-ascent structure of MPO, it differs in how Q-values are obtained and whether $\pi$ or $p_{\phi_k}$ is deployed as behavioral policy to gather trajectories. In standard MPO, $Q(s, a)$ is typically represented by a separate neural network (a critic) trained via temporal difference (TD) learning on the agent's own experience, relying on the Markov property to condition on a single state $s$ instead of the full history.

In contrast, PPI leverages the sequence model as a world model. The value $\hat{Q}_{p}(x_{\leq t}, a)$ is estimated via Monte Carlo rollouts performed within the sequence model itself. By sampling future trajectories $\tau_{>t}$ from $p_\phi(\cdot \mid x_{\leq t}, a)$, the agent evaluates the expected return of an action based on its internal representation of both the environment dynamics and the co-player's predicted responses. This allows PPI to benefit from the high-capacity temporal dependencies captured by the sequence model. Note that PPI is easily extendable toward learning an explicit Q-value function conditioned on full histories to amortize the cost of the MC rollouts and reduce variance.

\section{Theoretical Analysis of the Equilibrium Behavior of PPI Agents}
\label{app:predictive-equilibrea}

In this section, we analyze the theoretical properties of the Predictive Policy Improvement (PPI) algorithm. Unlike standard reinforcement learning, where agents optimize a policy against a fixed (or stationarily adapting) environment, PPI agents operate in a \textit{performative} loop: the agent's predictive model determines its policy, which determines the data distribution, which in turn is used to update the predictive model. This is closely related to the concept of ``performative prediction'' \citep{perdomo2020performative}, where the predictions of a model can affect the distribution of the very data it is trying to predict (with traffic prediction models being a prominent example).

We formalize this interaction and define the concept of a \textit{Predictive Equilibrium} (PE). We show that while a global pure-strategy equilibrium is not guaranteed to exist due to the non-convex nature of deep neural networks, a \textit{local} predictive equilibrium (consistent with gradient-based optimization) and a \textit{mixed} predictive equilibrium (randomized strategies) are guaranteed to exist under standard assumptions. Finally, we show that in the limit of a perfect world model, a predictive equilibrium corresponds to a subjective embedded equilibrium \citep{meulemans2025embedded}.

\subsection{Formal Setup}

Consider a game with $n$ agents. Each agent $i$ maintains a predictive sequence model $p_{\theta_i}(h^i)$ parameterized by $\theta_i \in \Theta_i$, where $h^i$ is a history $x^i_{\leq t}$ of arbitrary length $t$, and $\Theta_i$ is a compact metric space (e.g., a bounded subset of $\mathbb{R}^d$).

\paragraph{The Performative Loop.}
The PPI algorithm (Algorithm \ref{alg:seq_model_rl}) induces a closed-loop dependency between parameters and data:
\begin{enumerate}
    \item \textbf{Model induces Policy:} The agent derives a policy $\pi_{\theta_i}$ from its model $p_{\theta_i}$ via the policy improvement operator, defined in Eq.~\ref{eq:ppi} (the Boltzmann policy over Q-values estimated via rollout).
    \item \textbf{Policy induces Data:} When all agents interact using policies $\boldsymbol{\pi}_{\boldsymbol{\theta}} = \{\pi_{\theta_1}, \dots, \pi_{\theta_N}\}$, they induce a joint distribution over interaction histories $h$. We denote the true probability distribution of histories generated by the current joint configuration $\boldsymbol{\theta}$ as $\mathbb{P}(\cdot ; \boldsymbol{\theta})$.
    \item \textbf{Data induces Model:} The agent updates $\theta_i$ to minimize the Kullback-Leibler (KL) divergence between the observed distribution $\mathbb{P}(\cdot ; \boldsymbol{\theta})$ and its model $p_{\theta_i}$.
\end{enumerate}

\subsection{Predictive Equilibria}

A stable point of this training loop is a configuration where the model optimally predicts the data generated by the policy derived from that very model.

\begin{definition}[Global Predictive Equilibrium] \label{def:predictive_equilibrium}
A joint configuration $\boldsymbol{\theta}^* = (\theta^*_1, \dots, \theta^*_n)$ is a Global Predictive Equilibrium if, for all agents $i$:
\begin{align}
    \theta^*_i \in \argmin_{\theta_i \in \Theta_i} \text{KL}\left( \mathbb{P}(h_i ; \boldsymbol{\theta}^*) \,||\, p_{\theta_i}(h_i) \right)\,.
\end{align}
\end{definition}
Intuitively, at equilibrium, no agent can improve their world model given the behavior induced by the current joint models.

\paragraph{Challenges.} Proving the existence of a global PE is difficult because the map $\theta \mapsto \pi_\theta$ is complex and the resulting objective is generally non-convex. The ``argmin'' set may change discontinuously (mode hopping), preventing the application of standard fixed-point theorems. To address this, we define two relaxed solution concepts: \textit{Local} PE (relevant to gradient descent) and \textit{Mixed} PE.

\subsubsection{Local Predictive Equilibrium}
In practice, PPI agents update parameters via gradient descent. They do not find global minima but rather stationary points. Crucially, the update assumes that the data distribution is fixed (which can be interpreted as a ``stop-gradient'' on the environment dynamics).

\begin{definition}[Local Predictive Equilibrium]
Let $\Theta_i \subset \mathbb{R}^{d_i}$ be a compact, convex parameter space for each agent $i \in \mathcal{I}$. A joint configuration $\thetavec^* = (\theta^*_1, \dots, \theta^*_n) \in \prod_{i \in \mathcal{I}} \Theta_i$ is a Local Predictive Equilibrium if, for all agents $i \in \mathcal{I}$, the configuration satisfies the first-order stationarity condition with respect to their local loss, assuming the data generating process is fixed. Formally:
\begin{align}
\label{eq:local_predictive_equilibrium}
    \left\langle \nabla_{\theta_i} \kl{\mathbb{P}(h_i ; \thetavec^*)}{p_{\theta_i}(h_i)} \Big|_{\theta_i = \theta^*_i} , \phi_i - \theta^*_i \right\rangle \ge 0\,, \quad \forall \phi_i \in \Theta_i\,,\forall i\in\mathcal{I}\,,
\end{align}
where $\langle \cdot , \cdot \rangle$ denotes the standard inner product.
\end{definition}
This variational inequality definition corresponds precisely to the convergence criteria of projected gradient descent in the PPI algorithm. If $\theta^*_i$ lies in the interior of $\Theta_i$, Eq. \ref{eq:local_predictive_equilibrium} implies the standard condition
\begin{align}
\label{eq:local_predictive_equilibrium_interior_point}
    \nabla_{\theta_i} \kl{\mathbb{P}(h_i ; \thetavec^*)}{p_{\theta_i}(h_i)} \Big|_{\theta_i = \theta^*_i} = 0\,, \quad \forall i\in\mathcal{I}\,.
\end{align}

\begin{theorem}[Existence of Local Predictive Equilibrium]\label{theorem:existence_local_pe}
Assume $\Theta_i$ is a compact, convex subset of $\mathbb{R}^{d_i}$. Assume the mapping from parameters $\thetavec$ to the local gradient of the loss $G_i(\thetavec) = \nabla_{\vartheta} \kl{\mathbb{P}(\cdot; \thetavec)}{p_{\vartheta}}\big|_{\vartheta=\theta_i}$ is continuous. Then, there exists at least one Local Predictive Equilibrium.
\end{theorem}
\begin{proof}
We analyze the existence of a Local Predictive Equilibrium by framing it as a fixed-point problem.
Let $\mathcal{L}_i(\thetavec, \psi) = \kl{\mathbb{P}(h_i ; \thetavec)}{p_{\psi}(h_i)}$ denote the loss function for agent $i$, where the first argument $\thetavec$ determines the data distribution (fixed locally) and the second argument $\psi$ is the parameter being optimized.
Define the \textit{local gradient field} $G: \Theta \to \mathbb{R}^D$ (where $D = \sum_{i \in \mathcal{I}} d_i$) as the concatenation of the individual gradients:
$$ G(\thetavec) = \left( \nabla_{\psi} \mathcal{L}_1(\thetavec, \psi)\big|_{\psi=\theta_1}, \dots, \nabla_{\psi} \mathcal{L}_n(\thetavec, \psi)\big|_{\psi=\theta_n} \right) \,.$$
A Local Predictive Equilibrium is characterized by the variational inequality $\langle G(\thetavec^*), \phi - \thetavec^* \rangle \ge 0$ for all $\phi \in \Theta$, where $\Theta = \prod_{i \in \mathcal{I}} \Theta_i$.

We assume the parameter space $\Theta$ is a compact, convex subset of Euclidean space, and that the gradient field $G(\thetavec)$ is continuous. The continuity of $G$ follows naturally from the smoothness assumptions on the predictive models $p_\theta$ and the induced policy distributions.

Consider the map $T: \Theta \to \Theta$ defined by a projected gradient step:
$$ T(\thetavec) = \text{Proj}_{\Theta} \left( \thetavec - \eta G(\thetavec) \right)\,, $$
where $\eta > 0$ is a scalar step size and $\text{Proj}_{\Theta}$ is the Euclidean projection onto the set $\Theta$.
\begin{enumerate}
    \item \textbf{Compactness and Convexity:} By assumption, $\Theta$ is a compact and convex set.
    \item \textbf{Continuity:} The map $G$ is continuous by assumption. The projection operator $\text{Proj}_{\Theta}$ is non-expansive and thus continuous. Therefore, the composition $T$ is a continuous map from $\Theta$ to itself.
\end{enumerate}

By Brouwer's Fixed Point Theorem, there exists a point $\thetavec^* \in \Theta$ such that $T(\thetavec^*) = \thetavec^*$.
This fixed point condition implies:
$$ \thetavec^* = \text{Proj}_{\Theta} \left( \thetavec^* - \eta G(\thetavec^*) \right) \,.$$
By the standard property of the Euclidean projection onto a closed convex set, this equation holds if and only if:
$$ \langle \left( \thetavec^* - \eta G(\thetavec^*) \right) - \thetavec^*, \phi - \thetavec^* \rangle \le 0\,, \quad \forall \phi \in \Theta\,. $$
Simplifying the terms inside the inner product, we obtain:
$$ \langle - \eta G(\thetavec^*), \phi - \thetavec^* \rangle \le 0 \implies \langle G(\thetavec^*), \phi - \thetavec^* \rangle \ge 0\,, \quad \forall \phi \in \Theta\,. $$
This inequality is precisely the first-order stationarity condition defined in Eq. \ref{eq:local_predictive_equilibrium}, generalized to the joint parameter space $\Theta$. Therefore, the fixed point $\thetavec^*$ constitutes a Local Predictive Equilibrium, rigorously accommodating both interior stationary points and boundary solutions.
\end{proof}

\subsubsection{Mixed Predictive Equilibrium}
To guarantee the existence of an equilibrium without relying on local approximations, we can allow agents to randomize over model parameters. This is analogous to mixed strategies in game theory.

\begin{definition}[Mixed Predictive Equilibrium]
Let $\Delta_{\Theta_i}$ be the set of probability distributions over parameters $\Theta_i$. A Mixed Predictive Equilibrium is a tuple of distributions $\bm{\mu}^* = (\mu^*_1, \dots, \mu^*_n)$ such that for all $i \in \mathcal{I}$:
\begin{align}
    \mu^*_i \in \argmin_{\mu_i \in \Delta_{\Theta_i}} {\kl{\mathbb{P}(h_i ; \bm{\mu}^*)}{p_{\mu_i}(h_i)}}\,,
\end{align}
where $p_{\mu_i}(h_i) = \expect{\theta_i \sim \mu_i}{ p_{\theta_i}(h_i)}$, and $\mathbb{P}(h_i ; \bm{\mu}^*)$ is the distribution of histories generated when each agent $i$ follows the policy $\pi_{\mu_i}$ obtained by applying the policy improvement operator \eqref{eq:ppi} on $p_{\mu_i}$.
\end{definition}

\begin{theorem}[Existence of Mixed Predictive Equilibrium]
\label{thm:mixed-predictive-equilibria-existence}
Assume $\Theta_i$ is a compact metric space\footnote{It is worth noting that we do not require the convexity of $\Theta_i$ in \cref{thm:mixed-predictive-equilibria-existence}, we only need compactness.} and the map $(\thetavec,\theta_i') \mapsto \kl{\mathbb{P}(\cdot;\thetavec)}{p_{\theta_i'}}$ is continuous for every $i\in\mathcal{I}$. Furthermore, assume that $\kl{\mathbb{P}(h_i ; \bm{\mu})}{p_{\mu_i}(h_i)} < \infty$ for all $\bm{\mu} \in \Delta = \prod_{i \in \mathcal{I}} \Delta_{\Theta_i}$. Then a Mixed Predictive Equilibrium exists.
\end{theorem}
\begin{proof}
We prove existence by constructing a continuous map on the space of mixed strategies and applying a fixed-point theorem. Let $\Delta_{\Theta_i}$ be the space of Borel probability measures on the compact metric space $\Theta_i$. Endowed with the Wasserstein metric, $\Delta_{\Theta_i}$ is a compact, convex metric space. Let $\Delta = \prod_{i \in \mathcal{I}} \Delta_{\Theta_i}$ be the joint strategy space.

Since $\Delta_{\Theta_i}$ is a compact metric space, it is separable. We can fix a countable dense subset $D_i = \{\tilde{\mu}_{i,k} \}_{k=1}^\infty \subset \Delta_{\Theta_i}$.

We define the continuous advantage function $a_i: \Delta \times \Delta_{\Theta_i} \to \mathbb{R}_{\ge 0}$ as:
\begin{equation*}
    a_i(\bm{\mu}, \mu'_i) = \max\left\{ 0, \kl{\mathbb{P}(h_i ; \bm{\mu})}{p_{\mu_i}(h_i)} - \kl{\mathbb{P}(h_i ; \bm{\mu})}{p_{\mu'_i}(h_i)} \right\}.
\end{equation*}
Since $\kl{\mathbb{P}(h_i ; \bm{\mu})}{p_{\mu_i}(h_i)} < \infty$, the advantage function is well-defined and evaluates to a finite real number.

We now construct a transition map $T_i: \Delta \to \Delta_{\Theta_i}$. Define a finite measure $A_i(\bm{\mu})$ on $\Theta_i$ that places weights on the dense subset $D_i$ proportional to the advantage:
\begin{equation*}
    A_i(\bm{\mu}) = \sum_{k=1}^\infty 2^{-k} a_i(\bm{\mu}, \tilde{\mu}_{i,k}) \tilde{\mu}_{i,k}\,.
\end{equation*}
Let $A_i(\bm{\mu})(\Theta_i) = \sum_{k=1}^\infty 2^{-k} a_i(\bm{\mu}, \tilde{\mu}_{i,k})$ denote its total mass. We define $T_i(\bm{\mu})$ by mixing the current strategy $\mu_i$ with the improvement measure $A_i(\bm{\mu})$:
\begin{equation*}
    T_i(\bm{\mu}) = \frac{\mu_i + A_i(\bm{\mu})}{1 + A_i(\bm{\mu})(\Theta_i)}.
\end{equation*}

Since the mappings $\thetavec \mapsto \kl{\mathbb{P}(\cdot;\thetavec)}{p_{\mu'_i}}$ are continuous and the spaces are compact, $a_i$ is uniformly bounded and continuous in $\bm{\mu}$ with respect to the weak-* topology. Consequently, the joint map $T(\bm{\mu}) = (T_1(\bm{\mu}), \dots, T_n(\bm{\mu}))$ is a continuous function from the compact convex set $\Delta$ to itself. By Schauder's fixed-point theorem, there exists a fixed point $\bm{\mu}^* \in \Delta$ such that $T(\bm{\mu}^*) = \bm{\mu}^*$.

We now prove by contradiction that $\bm{\mu}^*$ is a Mixed Predictive Equilibrium. Let $$C_i = A_i(\bm{\mu}^*)(\Theta_i)\,.$$ From the fixed point condition $\mu^*_i = T_i(\bm{\mu}^*)$, we obtain:
\begin{equation*}
    \mu^*_i \left( 1 + C_i \right) = \mu^*_i + A_i(\bm{\mu}^*) \implies C_i \mu^*_i = A_i(\bm{\mu}^*).
\end{equation*}

Assume $\bm{\mu}^*$ is not a Mixed Predictive Equilibrium. Then, for some agent $i \in \mathcal{I}$, there exists a distribution 
$\hat{\mu}_i \in \Delta_{\Theta_i}$ such that $$\kl{\mathbb{P}(h_i ; \bm{\mu}^*)}{p_{\hat{\mu}_i}(h_i)} < \kl{\mathbb{P}(h_i ; \bm{\mu}^*)}{p_{\mu^*_i}(h_i)}\,.$$

Let
$$\epsilon := \kl{\mathbb{P}(h_i ; \bm{\mu}^*)}{p_{\mu^*_i}(h_i)} - \kl{\mathbb{P}(h_i ; \bm{\mu}^*)}{p_{\hat{\mu}_i}(h_i)} > 0\,.$$
By definition, the advantage of $\hat{\mu}_i$ is strictly positive: $a_i(\bm{\mu}^*, \hat{\mu}_i) = \epsilon > 0$.

Now since the mapping $(\thetavec,\theta_i') \mapsto \kl{\mathbb{P}(\cdot;\thetavec)}{p_{\theta_i'}}$ is continuous, it follows that the functional $\mu'_i \mapsto \kl{\mathbb{P}(h_i ; \bm{\mu}^*)}{p_{\mu'_i}(h_i)}$ is continuous on the compact metric space $\Delta_{\Theta_i}$, and hence the advantage function $a_i(\bm{\mu}^*, \cdot)$ is uniformly continuous. Therefore, there exists an open neighborhood $U \subset \Delta_{\Theta_i}$ containing $\hat{\mu}_i$ such that $a_i(\bm{\mu}^*, \mu'_i) > \epsilon/2$ for all $\mu'_i \in U$.

Since the set $D_i = \{\tilde{\mu}_{i,k}\}_{k=1}^\infty$ is dense in $\Delta_{\Theta_i}$, there exists an integer $K$ such that $\tilde{\mu}_{i,K} \in U$. Consequently, $a_i(\bm{\mu}^*, \tilde{\mu}_{i,K}) > \epsilon/2 > 0$. This strictly positive advantage guarantees that the total mass of the improvement measure is strictly positive:
\begin{equation*}
    C_i = A_i(\bm{\mu}^*)(\Theta_i) \ge 2^{-K} a_i(\bm{\mu}^*, \tilde{\mu}_{i,K}) > 0.
\end{equation*}

From the fixed-point condition $C_i \mu^*_i = A_i(\bm{\mu}^*)$, and knowing $C_i > 0$, we can express $\mu^*_i$ as an infinite convex combination of the basis measures in $D_i$:
\begin{equation*}
    \mu^*_i = \frac{1}{C_i} A_i(\bm{\mu}^*) = \sum_{k=1}^\infty w_k \tilde{\mu}_{i,k}\,,
\end{equation*}
where the weights $w_k = \frac{2^{-k} a_i(\bm{\mu}^*, \tilde{\mu}_{i,k})}{C_i} \ge 0$ sum to exactly $1$. 

Now, consider the expected predictive model under the mixed strategy $\mu^*_i$. By linearity of the expectation, we have:
\begin{equation*}
    p_{\mu^*_i}(h_i) = \expect{\theta_i \sim \mu^*_i}{p_{\theta_i}(h_i)} = \sum_{k=1}^\infty w_k \expect{\theta_i \sim \tilde{\mu}_{i,k}}{p_{\theta_i}(h_i)} = \sum_{k=1}^\infty w_k p_{\tilde{\mu}_{i,k}}(h_i)\,.
\end{equation*}

Because the Kullback-Leibler divergence is strictly convex with respect to its second argument, we can apply Jensen's inequality to the infinite convex combination:
\begin{align*}
    \kl{\mathbb{P}(h_i ; \bm{\mu}^*)}{p_{\mu^*_i}(h_i)} &= \kl{\mathbb{P}(h_i ; \bm{\mu}^*)}{\sum_{k=1}^\infty w_k p_{\tilde{\mu}_{i,k}}(h_i)} \\
    &\le \sum_{k=1}^\infty w_k \kl{\mathbb{P}(h_i ; \bm{\mu}^*)}{p_{\tilde{\mu}_{i,k}}(h_i)}\,.
\end{align*}

Crucially, by the definition of the advantage function and the construction of the weights $w_k$, any weight $w_k$ is strictly positive \textit{if and only if} the corresponding advantage $a_i(\bm{\mu}^*, \tilde{\mu}_{i,k}) > 0$. A strictly positive advantage exactly means that the evaluated measure achieves a strictly lower loss than the current state $\mu_i^*$:
\begin{equation*}
    \kl{\mathbb{P}(h_i ; \bm{\mu}^*)}{p_{\tilde{\mu}_{i,k}}(h_i)} < \kl{\mathbb{P}(h_i ; \bm{\mu}^*)}{p_{\mu^*_i}(h_i)}\,.
\end{equation*}

Since there is at least one weight $w_K > 0$ with an advantage bounded away from zero by $\epsilon/2$, substituting this strict upper bound into the sum over $k$ yields:
\begin{align*}
    \sum_{k=1}^\infty w_k \kl{\mathbb{P}(h_i ; \bm{\mu}^*)}{p_{\tilde{\mu}_{i,k}}(h_i)} &< \sum_{k=1}^\infty w_k \kl{\mathbb{P}(h_i ; \bm{\mu}^*)}{p_{\mu^*_i}(h_i)} \\
    &= \kl{\mathbb{P}(h_i ; \bm{\mu}^*)}{p_{\mu^*_i}(h_i)} \sum_{k=1}^\infty w_k \\
    &= \kl{\mathbb{P}(h_i ; \bm{\mu}^*)}{p_{\mu^*_i}(h_i)}\,.
\end{align*}

Combining the inequalities together, we arrive at the following absolute contradiction:
\begin{equation*}
    \kl{\mathbb{P}(h_i ; \bm{\mu}^*)}{p_{\mu^*_i}(h_i)} < \kl{\mathbb{P}(h_i ; \bm{\mu}^*)}{p_{\mu^*_i}(h_i)}\,.
\end{equation*}
Therefore, our initial assumption must be false. No such superior distribution $\hat{\mu}_i$ can exist, and the fixed point $\bm{\mu}^*$ is indeed a Mixed Predictive Equilibrium.
\end{proof}

An interesting corollary of the above theorem, is that if our model is convex in functional space, then there exists a pure global predictive equilibrium.

\begin{corollary}[Existence of Pure Predictive Equilibrium under Model Functional Convexity]
Consider the same assumptions as in \cref{thm:mixed-predictive-equilibria-existence}. Assume furthermore that for every agent $i \in \mathcal{I}$, the space of representable predictive models $\{ p_{\theta_i} \mid \theta_i \in \Theta_i \}$ is convex. That is, for every $\theta_i',\theta_i''\in\Theta_i$ and every $\alpha_i\in[0,1]$, there exists a pure parameter $\theta_i\in\Theta_i$ satisfying $p_{\theta_i} = \alpha_i p_{\theta'_i} + (1-\alpha_i) p_{\theta''_i}$.\footnote{We emphasize that we do not require convexity in the parameters, i.e., we do not require that $p_{\alpha_i\theta_i'+(1-\alpha_i)\theta_i''}=\alpha_i p_{\theta_i'}+(1-\alpha_i)p_{\theta_i''}$.} Under these conditions, a Global Predictive Equilibrium (in pure strategies) always exists.
\end{corollary}
\begin{proof}
From \cref{thm:mixed-predictive-equilibria-existence}, there exists a Mixed Predictive Equilibrium $\bm{\mu}^* = (\mu^*_1, \dots, \mu^*_n) \in \prod_{i \in \mathcal{I}} \Delta_{\Theta_i}$. To establish the existence of a pure Global Predictive Equilibrium, we will demonstrate that for any probability distribution $\mu_i \in \Delta_{\Theta_i}$, the model functional convexity assumption guarantees the existence of a pure parameter $\theta_i^* \in \Theta_i$ such that $p_{\theta_i^*} = p_{\mu_i} = \expect{\theta_i \sim \mu_i}{p_{\theta_i}}$. 

We first prove this claim for finitely supported measures. Let $\mu_i = \sum_{k=1}^m w_k \delta_{\theta_{i,k}}$ be a finitely supported probability measure on $\Theta_i$, where $w_k \ge 0$ and $\sum_{k=1}^m w_k = 1$. We proceed by induction on the support size $m$. The base case $m=1$ is trivial, as $p_{\mu_i} = p_{\theta_{i,1}}$. Assuming the claim holds for $m-1$, we can express $\mu_i$ (provided $w_m < 1$) as:
\begin{equation*}
    p_{\mu_i} = w_m p_{\theta_{i,m}} + (1 - w_m) \sum_{k=1}^{m-1} \frac{w_k}{1 - w_m} p_{\theta_{i,k}}\,.
\end{equation*}
By the inductive hypothesis, there exists a pure parameter $\tilde{\theta}_i \in \Theta_i$ such that $p_{\tilde{\theta}_i} = \sum_{k=1}^{m-1} \frac{w_k}{1 - w_m} p_{\theta_{i,k}}$. Applying the convexity assumption with $\alpha_i = w_m$, $\theta'_i = \theta_{i,m}$, and $\theta''_i = \tilde{\theta}_i$, there exists $\theta_i \in \Theta_i$ such that $p_{\theta_i} = w_m p_{\theta_{i,m}} + (1 - w_m) p_{\tilde{\theta}_i} = p_{\mu_i}$. Thus, the claim holds for all finitely supported measures.

Now, consider an arbitrary measure $\mu_i \in \Delta_{\Theta_i}$. Since the set of finitely supported measures is dense in $\Delta_{\Theta_i}$ under the weak-* topology, there exists a sequence of finitely supported measures $(\mu_i^{(m)})_{m=1}^\infty$ converging weakly to $\mu_i$.

Because the mapping $\theta_i \mapsto p_{\theta_i}(h_i)$ is continuous and bounded for any given $h_i$, the functional $\nu \mapsto p_\nu(h_i) = \int p_{\theta_i}(h_i) d\nu(\theta_i)$ is continuous with respect to the weak-* topology. Consequently, the sequence of expected models converges pointwise: $p_{\mu_i^{(m)}} \to p_{\mu_i}$ as $m \to \infty$.

From the inductive step, for each finitely supported measure $\mu_i^{(m)}$, there exists a corresponding pure parameter $\theta_i^{(m)} \in \Theta_i$ such that $p_{\theta_i^{(m)}} = p_{\mu_i^{(m)}}$. This constructs a sequence of pure parameters $(\theta_i^{(m)})_{m=1}^\infty$ in $\Theta_i$. Since $\Theta_i$ is a compact metric space, this sequence admits a convergent subsequence $(\theta_i^{(m_k)})_{k=1}^\infty$ that converges to some limit point $\theta_i^* \in \Theta_i$.

By the continuity of the map $\theta_i \mapsto p_{\theta_i}$, we find:
\begin{equation*}
    p_{\theta_i^*} = \lim_{k \to \infty} p_{\theta_i^{(m_k)}} = \lim_{k \to \infty} p_{\mu_i^{(m_k)}} = p_{\mu_i}\,.
\end{equation*}
Thus, for the Mixed Predictive Equilibrium $\bm{\mu}^*$, there exists a joint configuration of pure parameters $\thetavec^* = (\theta^*_1, \dots, \theta^*_n) \in \prod_{i \in \mathcal{I}} \Theta_i$ such that $p_{\theta^*_i} = p_{\mu^*_i}$ for all $i \in \mathcal{I}$.

It then follows that
\begin{align*}
    \theta^*_i \in \argmin_{\theta_i \in \Theta_i} \kl{\bbP(h_i; \thetavec^*)}{p_{\theta_i}(h_i)} \quad \forall i \in \mathcal{I}\,.
\end{align*}
This precisely satisfies the definition of a Global Predictive Equilibrium, proving its existence in pure strategies under these conditions.
\end{proof}

We remark that while the assumption of functional convexity is an idealization for finite-capacity networks, deep neural networks are universal function approximators; consequently, as model capacity increases, the space of representable distributions approximates the full convex set of valid probability measures, rendering the existence of a pure equilibrium an increasingly accurate approximation.

\subsection{Relationship to Nash Equilibria and Subjective Embedded Equilibria}

Finally, we connect the fixed points of the PPI algorithm to the standard solution concepts of game theory. In standard game theory, a Nash Equilibrium assumes that agents act optimally given a fixed environment, where the policies of co-players are independent of the focal agent's current action selection. In contrast, agents in the PPI framework act optimally with respect to an internal world model $p_{\theta_i}$ that estimates the joint distribution of future trajectories, thereby capturing potential reactive dependencies between the focal agent's actions and the co-players' responses.

This is closely related to the concept of ``Embedded Equilibria'', which characterizes the equilibrium behavior that emerges from such self-predictive dynamics:

\begin{definition}[Subjective Embedded Equilibrium] \citep{meulemans2025embedded}
A joint policy profile $\bm{\pi}^*$ and a set of internal sequence models $\{p^*_1, \dots, p^*_n\}$ constitute a Subjective Embedded Equilibrium if:
\begin{enumerate}
    \item \textbf{Subjective Optimality:} Each agent's policy $\pi^*_i$ is a strict best-response to its internal world model $p^*_i$.
    \item \textbf{On-Path Consistency:} Each agent's world model perfectly matches the true environment dynamics exclusively on the equilibrium path (the distribution of histories $\mathbb{P}^*$ genuinely generated by the joint policy $\bm{\pi}^*$).
\end{enumerate}
Crucially, a Subjective Embedded Equilibrium places no constraints on the accuracy of the agents' models regarding off-path counterfactuals (actions that are assigned zero probability under $\bm{\pi}^*$). Nevertheless, $\pi_i^*$ must be a best response with respect to $p_i^*$, and this takes into account counterfactual off-policy paths. In other words, according to the predictive model $p_i^*$, the agent $i$ will not get higher expected returns by deviating from $\pi_i^*$.
\end{definition}
We refer the reader to \citet{meulemans2025embedded} for further details about subjective embedded equilibria and their properties.

It turns out that if PPI agents converge to a fixed point for which their (predictive) world models are perfect, then the predictive equilibrium corresponds to a subjective embedded equilibrium. Let us first formalize the predictive equilibrium with perfect world models:

\begin{definition}[Perfect Predictive Equilibrium] \label{def:perfect_pe}
A Perfect Predictive Equilibrium is a configuration $\thetavec^*$ where the agents perfectly model the induced data distribution:
\begin{align}
    \kl{\mathbb{P}(h_i ; \thetavec^*)}{p_{\theta^*_i}(h_i)} = 0 \quad \forall i \in \mathcal{I}\,.
\end{align}
\end{definition}

\begin{theorem}[Perfect Predictive Equilibrium $\implies$ Subjective Embedded Equilibrium]\label{theorem:pe_ee}
Consider predictive agents using the policy improvement operator defined in Eq.~\ref{eq:ppi}, where $\pi_{\theta_i}(a_i | h_i) \propto p_{\theta_i}(a_i | h_i) \exp(\beta Q^{p_{\theta_i}}(h_i,a_i))$. If $\thetavec^*$ is a Perfect Predictive Equilibrium, then the resulting configuration is consistent with a Subjective Embedded Equilibrium.
\end{theorem}
\begin{proof}
At a Perfect Predictive Equilibrium, the condition $\kl{\mathbb{P}(\cdot ; \thetavec^*)}{p_{\theta^*_i}(\cdot)} = 0$ implies that the sequence model matches the true data distribution almost everywhere. Thus, on the equilibrium path, the prior action probability generated by the sequence model exactly matches the true behavioral policy: $p_{\theta^*_i}(a_i | h_i) = \pi_{\theta^*_i}(a_i | h_i)$. This immediately satisfies the On-Path Consistency condition.

Substituting $p_{\theta^*_i} = \pi_{\theta^*_i}$ into the policy improvement operator yields:
\begin{align*}
    \pi_{\theta^*_i}(a_i | h_i) &= \frac{1}{Z(h_i)} \pi_{\theta^*_i}(a_i | h_i) \exp(\beta Q^{p_{\theta^*_i}}(h_i,a_i))\,.
\end{align*}
For any action $a_i$ in the support of the policy (where $\pi_{\theta^*_i}(a_i | h_i) > 0$), we divide both sides by $\pi_{\theta^*_i}(a_i | h_i)$ to obtain:
\begin{align*}
    1 = \frac{1}{Z(h_i)} \exp(\beta Q^{p_{\theta^*_i}}(h_i,a_i)) \implies Q^{p_{\theta^*_i}}(h_i,a_i) = \frac{\ln Z(h_i)}{\beta}\,.
\end{align*}
Since $Z(h_i)$ is a normalizing constant independent of $a_i$, the expected return evaluated under the model must be identical for all actions played with positive probability.

Now consider any off-path action $a'_i$ not in the support of the policy (where $\pi_{\theta^*_i}(a'_i | h_i) = 0$). Because this action is never taken under the joint policy, the marginal probability $\mathbb{P}(h_i, a'_i ; \thetavec^*) = 0$. Consequently, the KL divergence places absolutely no constraints on the model's conditional predictions following $a'_i$. 

To formally verify Subjective Optimality, we demonstrate that there exists a valid completion of the sequence model's off-path conditional probabilities that justifies $\pi_{\theta^*_i}(a'_i | h_i) = 0$. Let $e_{min} = (o, r_{min})$ be an environment percept containing the minimal possible reward $r_{min}$. We define the model's off-path counterfactual completion as $p_{\theta^*_i}(e_{min} \mid h_i, a'_i) = 1$, assuming absorbing minimal rewards thereafter.

Evaluating the expected return under this completed subjective model yields $Q^{p_{\theta^*_i}}(h_i, a'_i) = V_{min}$, which is not larger than the on-path return $\frac{\ln Z(h_i)}{\beta}$. Because the policy operator is restricted by the prior $p_{\theta^*_i}(a'_i | h_i)$, which must evaluate to $0$ to satisfy the fixed point, the agent assigns exactly $0$ probability to the suboptimal deviation $a'_i$. Therefore, the agent is playing an exact, best-response to its subjective world model, fully satisfying the definition of a Subjective Embedded Equilibrium.
\end{proof}

\section{Software}
Experiments were implemented in Python together with the Google JAX~\citep{bradbury_jax_2018} framework, and the NumPy~\citep{harris_array_2020}, pandas~\citep{mckinney-proc-scipy-2010}, Matplotlib~\citep{hunter_matplotlib_2007}, seaborn~\citep{Waskom2021}, Flax~\citep{heek_flax_2024} and Optax~\citep{deepmind2020jax} packages.

\subsection{LLM usage}
We used Gemini 3 Pro for language editing and readability improvements during the preparation of this manuscript. We also used Gemini 3 Pro for providing additional details in the proof of Lemma C.2, which were afterwards checked by the authors.